\documentclass[a4paper]{article}

\usepackage[english]{babel}
\usepackage[utf8x]{inputenc}
\usepackage[T1]{fontenc}

\usepackage[preprint]{nips_2018}
\usepackage{lineno}

\usepackage{amsmath,amssymb,amsthm}
\usepackage{graphicx}
\usepackage{subcaption}
\usepackage{float}
\usepackage[toc,page]{appendix}
\usepackage[colorinlistoftodos]{todonotes}
\usepackage{hyperref}
\usepackage{todonotes}
\usepackage{natbib}
\usepackage{placeins}

\title{On Latent Distributions Without Finite Mean in Generative Models}

\author{Damian Leśniak\thanks{These two authors contributed equally}\qquad Igor Sieradzki\footnotemark[1]\qquad Igor Podolak \\ Jagiellonian University}

\newtheorem{observation}{Observation}[section]

\begin{document}
\maketitle

\begin{abstract}
We investigate the properties of multidimensional probability distributions in the context of latent space prior distributions of implicit generative models.
Our work revolves around the phenomena arising while decoding linear interpolations between two random latent vectors -- regions of latent space in close proximity to the origin of the space are sampled causing distribution mismatch.
We show that due to the Central Limit Theorem, this region is almost never sampled during the training process.
As a result, linear interpolations may generate unrealistic data and their usage as a tool to check quality of the trained model is questionable.
We propose to use multidimensional Cauchy distribution as the latent prior. Cauchy distribution does not satisfy the assumptions of the CLT and has a number of properties that allow it to work well in conjunction with linear interpolations. We also provide two general methods of creating non-linear interpolations that are easily applicable to a large family of common latent distributions. Finally we empirically analyze the quality of data generated from low-probability-mass regions for the DCGAN model on the CelebA dataset. 
\end{abstract}

\section{Introduction}

\subsection{Motivation and related work}

Generative latent variable models have grown to be a very popular research topic, with Variational Auto-Encoders (VAEs)~\cite{kingma2013auto} and Generative Adversarial Networks (GANs)~\cite{goodfellow2014generative} gaining a lot of research interest in the last few years. VAEs use a stochastic \textit{encoder} network to embed input data in a typically lower dimensional space, using a conditional probability distribution $p(\mathbf{z}|x)$ over possible latent space codes $z \in \mathbb{R}^D$. A stochastic \textit{decoder} network is then used to reconstruct the original sample. GANs, on the other hand, use a \textit{generator} network that creates data samples from noise samples $z \sim p(\mathbf{z})$, where $p(\mathbf{z})$ is a fixed prior distribution, and train a \textit{discriminator} network jointly to distinguish between real and "fake" (i.e. generated) data.

Both of those model families use a specific prior distribution on the latent space. In those models the latent codes aim to "explain" the underlying features of the real distribution $p(\mathbf{x})$ without explicit access to it. One would expect a~well-trained probabilistic model to encode the properties of the data. Typical priors for those latent codes are the multidimensional standard Normal distribution $\mathcal{N}(\mathbf{0}, \mathbf{I})$ or uniform distribution on a hypercube~$[-1, 1]^D$.

A \textit{linear interpolation} between two latent vectors $x_1, x_2$ is formally defined as a function
\[
f^L_{x_1,x_2}: [0,1]\ni \lambda \mapsto (1-\lambda)x_1 + \lambda x_2,
\] 
which may be understood as a traversal along the shortest path between these two endpoints. We are interested in decoding data for several values $\lambda$ and inspecting how smooth the transition between the decoded data points is.
Linear interpolations were utilized in previous work on generative models, mainly to show that the learned models do not overfit~\cite{kingma2013auto,goodfellow2014generative,dumoulin2016adversarially} and that the latent space is able to capture the semantic content of the data~\cite{radford2015unsupervised,donahue2016adversarial}. Linear interpolations can also be thought of as special case of vector algebra in the code space, similarly to the work done in word embeddings~\cite{mikolov2013efficient}.

While considered useful, linear interpolations used in conjunction with the most popular latent distributions are prone to traverse low probability mass regions. In high dimensions norms of vectors drawn from the latent distribution are concentrated around a certain value. Thus latent vectors are found near the surface of a sphere which results in the latent space distribution resembling a \emph{soap bubble}~\cite{ferenc}. This is explained using the Central Limit Theorem (CLT), which we show in~\ref{cha:cauchy}. Linear interpolations pass through inside of the sphere with high enough probability to drastically change the distribution of interpolated points in comparison to the prior distribution. This was reported~\cite{kilcher2017semantic} to result in flawed data generation.

Some approaches to counteract this phenomena were proposed: \citep{white2016sampling}
recommended using spherical interpolations to avoid traversing unlikely regions; \cite{agustsson2017optimal} suggest normalizing the norms of the points along the interpolation to match the prior distribution; \citep{kilcher2017semantic} propose using a modified prior distribution, which saturates the origin of the latent space. \citep{arvanitidis2017latent} gives an interesting discussion on latent space traversal using the theory of Riemannian spaces.

\subsection{Main contributions}

Firstly, we propose to use the Cauchy distribution as the prior in generative models.  This results in points along linear interpolations being distributed identically to those sampled from the prior. This is possible because Cauchy distributed noise does not satisfy the assumptions of the CLT.

Furthermore, we present two general ways of defining non-linear interpolations for a given latent distribution. Similarly, we are able to force points along interpolations to be distributed according to the prior.

Lastly, we show that the DCGAN \cite{radford2015unsupervised} model on the CelebA \cite{liu2015deep} dataset is able to generate sensible images from the region near the supposedly "empty" origin of the latent space. This is contrary to what has been reported so far and we further empirically investigate this result by evaluating the model trained with specific pathological distributions.

\subsection{Notations and mathematical conventions}

The normal distribution with mean $\mu$ and variance $\sigma^2$ is denoted by $\mathcal{N}(\mu, \sigma^2)$,
the uniform distribution on the interval $[a,b]$ is denoted by $\mathcal{U}(a,b)$,
and the Cauchy distribution with location $\mu$ and scale $\gamma$ is denoted by $\mathcal{C}(\mu, \gamma)$.
If not stated otherwise, the normal distribution has mean zero and variance one,
the uniform distribution is defined on the interval $[-1, 1]$,
and the Cauchy distribution has location zero and scale one.

The dimension of the latent space is denoted by $D$.

Multidimensional random variables are written in bold, e.g. $\mathbf{Z}$.
Lower indices denote coordinates of multidimensional random variables, e.g. $\mathbf{Z} = (Z_1, \ldots, Z_D)$.
Upper indices denote independent samples from the same distribution, e.g. $\mathbf{Z}^{(1)}, \mathbf{Z}^{(2)}, \ldots, \mathbf{Z}^{(n)}$.
If not stated otherwise, $D$-dimensional distributions are defined as products of $D$ one-dimensional independent equal distributions.

The norm used in this work is always the Euclidean norm.

\section{The Cauchy distribution}\label{cha:cauchy}
Let us assume that we want to train a generative model which has a $D$-dimensional latent space and a fixed latent probability distribution defined by random variable $\mathbf{Z} = (Z_1, Z_2, \ldots, Z_D)$. $Z_1, Z_2, \ldots, Z_D$ are the independent marginal distributions, and let $Z$ denote a one-dimensional random variable distributed identically to every $Z_j$, where $j=1,\ldots,D$.

For example, if $Z \sim \mathcal{U}(-1,1)$, then $\mathbf{Z}$ is distributed uniformly on the hypercube $[-1,1]^D$; if $Z \sim \mathcal{N}(0,1)$ , then $\mathbf{Z}$ is distributed according to the $D$-dimensional normal distribution with mean $\mathbf{0}$ and identity covariance matrix.

In the aforementioned cases we observe the so-called \textit{soap bubble} phenomena -- the values sampled from $\mathbf{Z}$ are concentrated close to a $(D-1)$-dimensional sphere, contrary to the low-dimensional intuition.

\begin{observation}
Let us assume that $Z^2$ has finite mean $\mu$ and finite variance $\sigma^2$. Then $\|\mathbf{Z}\|$ approximates the normal distribution with mean $\sqrt{D\mu}$ and variance $\dfrac{\sigma^2}{4\mu}$.
\end{observation}

\begin{proof}[Sketch of proof]
Recall that $\|\mathbf{Z}\|^2 = Z_1^2 + \ldots + Z_D^2$.
If $Z_1, \ldots, Z_D$ are independent and distributed identically to $Z$, then $Z_1^2, \ldots, Z_D^2$ are independent and distributed identically to $Z^2$.
Using the central limit theorem we know that for large $D$
\[
\sqrt{D}(\dfrac{Z_1^2 + \ldots + Z_D^2}{D} - \mu) \simeq \mathcal{N}(0, \sigma^2)
\]
from which it follows
\[
(\dfrac{Z_1^2 + \ldots + Z_D^2}{D}) \simeq \mathcal{N}(\mu, \dfrac{\sigma^2}{D}),
\]
and thus we can approximate the squared norm of $\mathbf{Z}$ as
\[
\|\mathbf{Z}\|^2 = Z_1^2 + \ldots + Z_D^2 \simeq \mathcal{N}(D\mu, D\sigma^2).
\]
Due to the nature of the convergence in distribution dividing or multiplying both sides by factors $\sqrt{D}$ or $D$ that tend to infinity does not break the approximation.

The final step is to take the square root of both random variables. In proximity of $D\mu$, square root behaves approximately like scaling with constant $(2\sqrt{D\mu})^{-1}$. Additionally, $\mathcal{N}(D\mu, D\sigma^2)$ has \textit{width} proportional to $\sqrt{D}$, so we may apply affine transformation to the normal distribution to approximate the square root for large D, which in the end gives us:
\[
\|\mathbf{Z}\| \simeq \mathcal{N}(\sqrt{D\mu}, \dfrac{\sigma^2}{4\mu}).
\]
\end{proof}

An application of this observation to the two most common latent space distributions: 

\begin{itemize}

\item[$\circ$] if $Z \sim \mathcal{N}(0,1)$, then $Z^2$ has moments $\mu=1$, $\sigma^2=2$, thus $\|\mathbf{Z}\|\simeq \mathcal{N}(\sqrt{D}, 1)$,

\item[$\circ$] if $Z \sim \mathcal{U}(-1,1)$, then $Z^2$ has moments $\mu=\dfrac{1}{3}$, $\sigma^2=\dfrac{4}{45}$, thus $\|\mathbf{Z}\|\sim \mathcal{N}(\sqrt{\dfrac{D}{3}}, \dfrac{1}{15})$.

\end{itemize}

It is worth noting that the variance of the norm does not depend on $D$, which means that the distribution does not converge to the uniform distribution on the sphere of radius $\sqrt{D\mu}$. Another fact worth noting is the observation that the $D$-dimensional normal distribution with identity covariance matrix is isotropic, hence this distribution resembles the uniform distribution on a sphere. On the other hand, the uniform distribution on the hypercube $[-1,1]^D$ is concentrated in close proximity to the surface of the sphere, but has regions of high density corresponding to directions defined by the hypercube's vertices.





Now let us assume that we want to randomly draw two latent samples and interpolate linearly between them. We denote the two independent draws by $\mathbf{Z}^{(1)}$ and $\mathbf{Z}^{(2)}$. Let us examine the distribution of the random variable $\overline{\mathbf{Z}}:= \dfrac{\mathbf{Z}^{(1)} + \mathbf{Z}^{(2)}}{2}$. $\overline{\mathbf{Z}}$ is the distribution of the middle points of a linear interpolation between two vectors drawn independently from $\mathbf{Z}$. If the generative model was trained on noise sampled from $\mathbf{Z}$ and if the distribution of $\mathbf{Z}$ differs from $\overline{\mathbf{Z}}$, then data decoded from samples drawn from $\overline{\mathbf{Z}}$ \emph{might} be unrealistic, as such samples were never seen during training. One way to prevent this issue is to find $\mathbf{Z}$ such that $\overline{\mathbf{Z}}$ is distributed identically to $\mathbf{Z}$.

\begin{observation}\label{obs:mean}
If $\mathbf{Z}$ has a finite mean and $\mathbf{Z}, \overline{\mathbf{Z}}$ are identically distributed, then $\mathbf{Z}$ must be concentrated at a single point.
\end{observation}

\begin{proof}[Sketch of proof]
Using induction on $n$ we can show that for all $n\in\mathbb{N}$ the average of $2^n$ independent samples from $\mathbf{Z}$ is distributed equally to $\mathbf{Z}$. On the other hand if $n\to\infty$, then the average distribution tends to $\mathbb{E}[\mathbf{Z}]$. Thus $\mathbf{Z}$ must be concentrated at $\mathbb{E}[\mathbf{Z}]$.
\end{proof}

There have been attempts to find $\mathbf{Z}$ with finite mean such that $\overline{\mathbf{Z}}$ is at least similar to $\mathbf{Z}$~\cite{kilcher2017semantic}, where similarity was measured with Kullback-Leibler divergence between the distributions. We extend this idea by using a specific distribution that has no finite mean, namely the multidimensional Cauchy distribution.

Let us start with a short review of useful properties Cauchy distribution in of one-dimensional case. Let $C \sim \mathcal{C}(0,1)$. Then:
\begin{enumerate}
  \item The probability density function of $C$ is equal to $\dfrac{1}{\pi(1+x^2)}$.
  
  \item For $C$ all moments of order greater than or equal to one are undefined. The location parameter should not be confused with the mean.
  
  \item If $C^{(1)}$ and $C^{(2)}$ are independent and distributed identically to $C$, then $\dfrac{C^{(1)} + C^{(2)}}{2}$ is distributed identically to $C$. Furthermore, if $\lambda\in[0,1]$, then $\lambda C^{(1)} + (1-\lambda)C^{(2)}$ is also distributed identically to $C$.
  
  \item If $C^{(1)}, \ldots, C^{(n)}$ are independent and distributed identically to $C$, and $\lambda_1, \ldots, \lambda_n\in[0,1]$ with $\lambda_1 + \ldots + \lambda_n = 1$, then $\lambda_1 C^{(1)}, \ldots, \lambda_n C^{(n)}$ is distributed identically to $C$.
\end{enumerate}

Those are well-known facts about the Cauchy distribution, and proving them is a common exercise in statistics textbooks. However, according to our best knowledge, the Cauchy distribution has never been used for in the context of generative models. With this in mind, the most important take-away is the following observation:


\begin{observation}\label{obs:cauchy_dist}
If $\mathbf{Z}$ is distributed according to the $D$-dimensional Cauchy distribution, then a~linear interpolation between any number of latent points does not change the distribution.
\end{observation}
\begin{proof}[Sketch of proof]
Let $C \sim \mathcal{C}(0,1)$ and $\mathbf{Z}=(Z_1,\ldots, Z_D)$. The variables $Z_1,\ldots, Z_D$ are independent and distributed equally to $C$. If $\mathbf{Z}^{(1)}, \ldots, \mathbf{Z}^{(n)}$ are independent and distributed identically to $\mathbf{Z}$, and $\lambda_1, \ldots, \lambda_n\in[0,1]$ with $\lambda_1 + \ldots + \lambda_n = 1$ are fixed, then 
$\lambda_1 Z_j^{(1)}, \ldots, \lambda_n Z_j^{(n)}$ is distributed equally to $Z_j$ for $j=1,\ldots,D$, thus $\lambda_1 \mathbf{Z}^{(1)}, \ldots, \lambda_n \mathbf{Z}^{(n)}$ is distributed equally to $\mathbf{Z}$.
\end{proof}

We observed that the normal and uniform distributions are concentrated around a sphere with radius proportional to $\sqrt{D}$. On the other hand, the multidimensional Cauchy distribution \textit{fills} the latent space. It should be noted that for the $D$-dimensional Cauchy distribution the region near the origin of the latent space is empty -- similarly to the normal and uniform distributions.

Figure~\ref{fig:density} shows a comparison between approximations of density functions of $\|\mathbf{Z}\|$ for multidimensional normal, uniform and Cauchy distributions, and a distribution proposed by~\cite{kilcher2017semantic}.

\begin{figure}[htb]
  \centering
  \includegraphics[width=\textwidth]{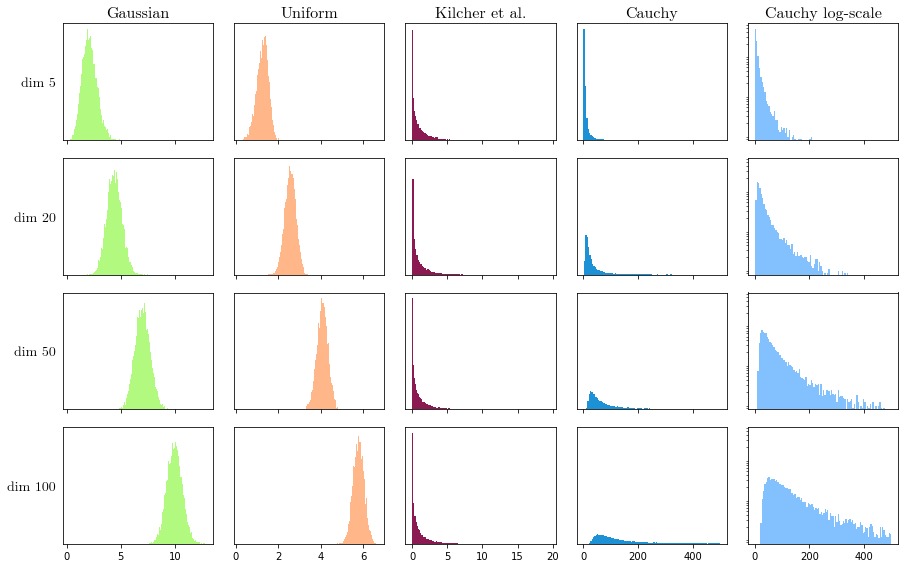}
  \caption{Comparison of approximate distribution of Euclidean norms for $10,000$ samples with increasing dimensionality from different probability distributions.}
  \label{fig:density}
\end{figure}

\begin{figure}[htb]
  \centering
  \includegraphics[width=0.9\textwidth]{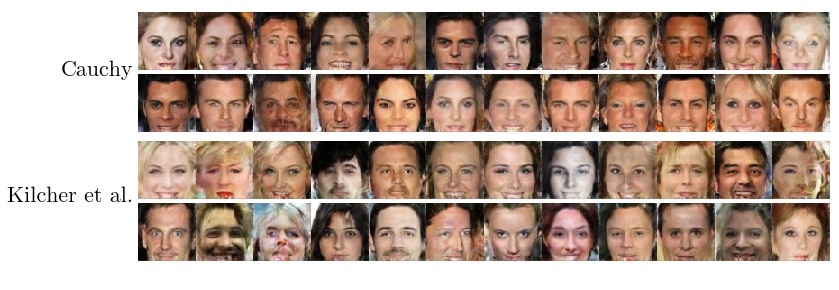}
  \caption{Comparison of samples from DCGAN trained on the Cauchy distribution and one trained on the distribution proposed by \cite{kilcher2017semantic}.}
  \label{fig:cauchy_samples}
\end{figure}

The one-dimensional Cauchy distribution has heavy tails, hence we can expect that one of $\mathbf{Z^{(i)}}$ coordinates will usually be sufficiently larger (by absolute value) than the others. This could potentially have negative impact on training of the GAN model, but we did not observe such difficulties. However, there is an obvious trade-off with using a distribution with heavy tails, as there will always be a number of samples with high enough norm. For those samples the generator will not be able to create sensible data points. 
A particular result of choosing a Cauchy distributed prior in GANs is the fact that during inference there will always be a number of "failed" generated data points due to latent vectors being sampled from the tails. Some of those faulty examples are presented in the appendix~\ref{cha:app_cauchy}. Figure~\ref{fig:cauchy_samples} shows a set of samples from the DCGAN model trained on the CelebA dataset using the Cauchy and distribution from~\cite{kilcher2017semantic} and Figure~\ref{fig:cauchy_inter} shows linear interpolation on those two models.

\begin{figure}[htb]
  \centering
  \includegraphics[width=0.9\textwidth]{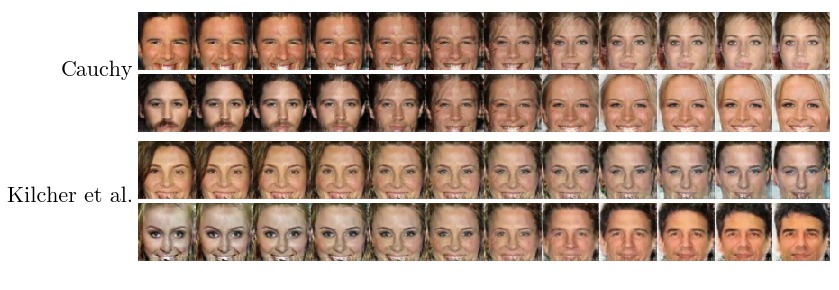}
  \caption{Comparison of linear interpolations from DCGAN trained on the Cauchy distribution and one trained on the distribution proposed by \cite{kilcher2017semantic}.}
  \label{fig:cauchy_inter}
\end{figure}

\section{Interpolations}\label{cha:interpolations}

In this section we list current work on interpolations in high dimensional latent spaces in generative models. We present two methods that perform well with noise priors with finite first moments, i.e.~the mean. Again, we define a linear interpolation between two points $x_1$ and $x_2$ as a function 

\[
f^L_{x_1,x_2}:[0,1]\ni\lambda\mapsto  (1-\lambda) x_1 + \lambda x_2.
\]

In some cases we will use the term \textit{interpolation} for the \textit{image} of the function, as opposed to the function itself. 
We will list four properties an interpolation can have that we believe are important in the context of generative models:

\textbf{Property 1.} The interpolation should be continuous with respect to $\lambda, x_1$ and $x_2$.

\textbf{Property 2.} For every $x_1, x_2$ the interpolation $f_{x_1,x_2}$ should represent the shortest path between the two endpoints.

\textbf{Property 3.} If two points $x_1', x_2'$ are a in the interpolation between $x_1$ and $x_2$, then the whole interpolation from $x_1'$ to $x_2'$ should be included in the interpolation between $x_1$ and $x_2$.

\textbf{Property 4.} If $\mathbf{Z}$ defines a distribution on the $D$-dimensional latent space and $\mathbf{Z}^{(1)}, \mathbf{Z}^{(2)}$ are independent and distributed identically to $\mathbf{Z}$, then for every $\lambda\in[0,1]$ the random variable $f_{\mathbf{Z}^{(1)}, \mathbf{Z}^{(2)}}(\lambda)$ should be distributed identically to $\mathbf{Z}$.

The first property enforces that an interpolation should not make any \textit{jumps} and that interpolations between pairs of similar endpoints should also be similar to each other. The second one is purposefully ambiguous. In absence of any additional information about the latent space it feels natural to use the Euclidean metric and assume that only the linear interpolation has this property. There has been some work on equipping the latent space with a stochastic Riemannian metric \cite{arvanitidis2017latent} that additionally depends on the generator function. With such a metric the shortest path can be defined using geodesics. The third property is closely associated  with the second one and codifies common-sense intuition about shortest paths. The fourth property is in our minds the most important desideratum of the linear interpolation, similarly to what \cite{kilcher2017semantic} stated. To understand these properties better, we will now analyze the following interpolations.

\subsection{Linear interpolation}

The \emph{linear interpolation} is defined as

\[
f^L_{x_1,x_2}(\lambda) = (1-\lambda) x_1 + \lambda x_2.
\]

It obviously has properties 1-3. Satisfying property 4 is impossible for the most commonly used probability distributions, as they have finite mean, which was shown in observation~\ref{obs:mean}.

\subsection{Spherical linear interpolation}
As in~\cite{shoemake1985animating,white2016sampling}, the \emph{spherical linear interpolation} is defined as

\[
f^{SL}_{x_1,x_2}(\lambda) = \dfrac{\sin {[(1-\lambda)\Omega}]}{\sin \Omega} x_1 + \dfrac{\sin [\lambda\Omega]}{\sin \Omega} x_2,
\]

where $\Omega$ is the angle between vectors $x_1$ and $x_2$.

This interpolation is continuous nearly everywhere (with the exception of antiparallel endpoint vectors) and satisfies property 3. It satisfies property 2 in the following sense: if vectors $x_1$ and $x_2$ have the same length $R$, then the interpolation corresponds to a geodesic on the sphere of radius $R$. Furthermore:

\begin{observation}
Property 4 is satisfied if $\mathbf{Z}$ has uniform distribution on the zero-centered sphere of radius $R>0$.
\end{observation}

\begin{proof}[Sketch of proof]
Let $\lambda\in[0,1]$ and let $f_{\mathbf{Z}^{(1)}, \mathbf{Z}^{(2)}}(\lambda)$ be concentrated on the zero-centered sphere. The distribution of all pairs sampled from $\mathbf{Z}^{(1)}\times \mathbf{Z}^{(2)}$ is identical to the product of two uniform distributions on the sphere, thus invariant to all isometries of the sphere. Then $f_{\mathbf{Z}^{(1)}, \mathbf{Z}^{(2)}}(\lambda)$ also must be invariant to all isometries, and the only probability distribution having this property is the uniform distribution.
\end{proof}

\subsection{Normalized interpolation} 
Introduced in \cite{agustsson2017optimal}, the \emph{normalized\footnote{Originally referred to as \textit{distribution matched}.} interpolation} is defined as

\[
f^N_{x_1,x_2}(\lambda) = \dfrac{(1-\lambda) x_1 + \lambda x_2}{\sqrt{(1-\lambda)^2 + \lambda^2}}.
\]

It satisfies property 1, but neither property 2 nor 3, which can be easily shown in the extreme case of $x_1 = x_2$. As for property 4:

\begin{observation}
The normalized interpolation satisfies property 4 if\, $\mathbf{Z} \sim \mathcal{N}(\mathbf{0},\mathbf{I})$. 
\end{observation}

\begin{proof}[Sketch of proof]
Let $\lambda\in[0,1]$. The random variables $\mathbf{Z}^{(1)}$ and $\mathbf{Z}^{(2)}$ are both distributed according to $\mathcal{N}(\mathbf{0},\mathbf{I})$. Then, using elementary properties of the normal distribution:

\[
\dfrac{(1-\lambda)\mathbf{Z}^{(1)} + \lambda\mathbf{Z}^{(2)}}{\sqrt{(1-\lambda)^2 + \lambda^2}} \sim \mathcal{N}(\mathbf{0},\mathbf{I}).
\]

\end{proof}

If vectors $x_1$ and $x_2$ are orthogonal and have equal length, then this interpolation is equal to the spherical linear interpolation from the previous section.

\subsection{Cauchy-linear interpolation.}

Here we present a general way of designing interpolations that satisfy properties 1, 3, and 4. Let:
\begin{itemize}
\item[$\circ$] $L$ be the $D$-dimensional latent space,
\item[$\circ$] $\mathbf{Z}$ define the probability distribution on the latent space,
\item[$\circ$] $\mathbf{C}$ be distributed according to the $D$-dimensional Cauchy distribution on $L$,
\item[$\circ$] $K$ be a subset of $L$ such that all mass of $\mathbf{Z}$ is concentrated on this set,
\item[$\circ$] $g: L\to K$ be a bijection such that $g(\mathbf{C})$ be identically distributed as $\mathbf{Z}$ on $K$.
\end{itemize}

Then for $x_1,x_2\in K$ we define the \emph{Cauchy-linear interpolation as}
\[
f^{CL}_{x_1,x_2}(\lambda) = g\big((1-\lambda)g^{-1}(x_1) + \lambda g^{-1}(x_2)\big).
\]
In other words, for endpoints $x_1, x_2 \sim \mathbf{Z}$:
\begin{enumerate}
\item Transform $x_1$ and $x_2$ using $g^{-1}$.
\item Linearly interpolate between the transformations to get $\overline{x} = (1-\lambda)g^{-1}(x_1) + \lambda g^{-1}(x_2)$ for all $\lambda \in [0,1]$.
\item Transform $\overline{x}$ back to the original space using $g$. 
\end{enumerate}

With some additional assumptions we can define $g$ as $CDF_C^{-1} \circ CDF_Z$, where $CDF_C^{-1}$ is the inverse of the cumulative distribution function (CDF) of the Cauchy distribution, and $CDF_Z$ is the CDF of the original distribution $\mathbf{Z}$. If additionally $\mathbf{Z}$ is distributed identically to the product of $D$ independent one-dimensional distributions, then we can use this formula coordinate-wise.

\begin{observation}
With the above assumptions the Cauchy-linear interpolation satisfies property 4.
\end{observation}

\begin{proof}[Sketch of proof]
Let $\lambda\in[0,1]$.
First observe that $g^{-1}(\mathbf{Z}^{(1)})$ and $g^{-1}(\mathbf{Z}^{(2)})$ are independent and distributed identically to $\mathbf{C}$.
Likewise, $(1-\lambda) g^{-1}(\mathbf{Z}^{(1)}) + \lambda g^{-1}(\mathbf{Z}^{(2)}) \sim \mathbf{C}$. By the assumption on $g$ we have 
$g((1-\lambda)g^{-1}(\mathbf{Z}^{(1)}) + \lambda g^{-1}(\mathbf{Z}^{(1)})) \sim \mathbf{Z}$.

\end{proof}

\subsection{Spherical Cauchy-linear interpolation.}

We might want to enforce the interpolation to have some other desired properties. For example: to behave exactly as the spherical linear interpolation, if only the endpoints have equal norm. For that purpose we require additional assumptions. Let:
\begin{itemize}
\item[$\circ$] $\mathbf{Z}$ be distributed isotropically,
\item[$\circ$] $C$ be distributed according to the one-dimensional Cauchy distribution,
\item[$\circ$] $g: \mathbb{R}\to (0,+\infty)$ be a bijection such that $g(C)$ is distributed identically as $\|\mathbf{Z}\|$ on $(0,+\infty)$.
\end{itemize}

Then we can modify the spherical linear interpolation formula to define what we call the \emph{spherical Cauchy-linear interpolation}:

\begin{equation*} 
  f_{x_1,x_2}(\lambda) = \Big(\dfrac{\sin {[(1-\lambda)\Omega}]}{\sin \Omega} \dfrac{x_1}{\|x_1\|} + \dfrac{\sin [\lambda\Omega]}{\sin \Omega} \dfrac{x_2}{\|x_2\|}\Big)\Big[g\big((1-\lambda)g^{-1}(\|x_1\|) + \lambda g^{-1}(\|x_2\|)\big)\Big],
\end{equation*}

where $\Omega$ is the angle between vectors $x_1$ and $x_2$. In other words:
\begin{enumerate}
\item Interpolate the \textit{directions} of latent vectors using spherical linear interpolation.
\item Interpolate the \textit{norms} using \textit{Cauchy-linear interpolation} from previous section.
\end{enumerate}

Again, with some additional assumptions we can define $g$ as $CDF_C^{-1} \circ CDF_{\|\mathbf{Z}\|}$. For example: let $\mathbf{Z}$ be a $D$-dimensional normal distribution with zero mean and identity covariance matrix. Then $\|\mathbf{Z}\| \sim \sqrt{\chi^2_D}$ and 

\[
CDF_{\sqrt{\chi^2_D}}(x) = CDF_{\chi^2_D}(x^2) = \dfrac{1}{\Gamma(D/2)}\,\gamma\left(\dfrac{D}{2}, \dfrac{x^2}{2}\right), \text{ for every } x \geq 0.
\]

Thus we set $g(x) = (CDF_C^{-1} \circ CDF_{\chi^2_D})(x^2)$, with $g^{-1}(x) = \sqrt{(CDF_{\chi^2_D}^{-1} \circ CDF_C)(x)}$.

\begin{observation}
With the assumptions as above, the spherical Cauchy-linear interpolation satisfies property 4.
\end{observation}

\begin{proof}[Sketch of proof]
We will use the fact that two isotropic probability distributions are equal if distributions of their euclidean norms are equal. The following holds:
\begin{itemize}
  \item[$\circ$] All the following random variables are independent: $ \dfrac{\mathbf{Z}^{(1)}}{\|\mathbf{Z}^{(1)}\|}, \dfrac{\mathbf{Z}^{(2)}}{\|\mathbf{Z}^{(2)}\|}, \|\mathbf{Z}^{(1)}\|, \|\mathbf{Z}^{(2)}\|$.
  
  \item[$\circ$] $\|\mathbf{Z}^{(1)}\|$ and $\|\mathbf{Z}^{(2)}\|$ are both are distributed identically to $\|\mathbf{Z}\|$.

\item[$\circ$] $\dfrac{\mathbf{Z}^{(1)}}{\|\mathbf{Z}^{(1)}\|}$ and $\dfrac{\mathbf{Z}^{(2)}}{\|\mathbf{Z}^{(2)}\|}$ are both distributed uniformly on the sphere of radius $1$.

\end{itemize}

Let $\lambda\in[0,1]$. Note that

\begin{equation}\label{eg:sph_lin}
\dfrac{\sin {[(1-\lambda)\Omega}]}{\sin \Omega} \dfrac{\mathbf{Z}^{(1)}}{\|\mathbf{Z}^{(1)}\|} + \dfrac{\sin [\lambda\Omega]}{\sin \Omega} \dfrac{\mathbf{Z}^{(2)}}{\|\mathbf{Z}^{(2)}\|}
\end{equation}
    
is uniformly distributed on the sphere of radius $1$, which is a property of spherical linear interpolation. The norm of $f_{\mathbf{Z}^{(1)},\mathbf{Z}^{(2)}}(\lambda)$ is distributed according to $g((1-\lambda)g^{-1}(\|\mathbf{Z}^{(1)}\|) + \lambda g^{-1}(\|\mathbf{Z}^{(2)}\|))$ which is independent of (\ref{eg:sph_lin}). Thus, we have shown that $f_{\mathbf{Z}^{(1)},\mathbf{Z}^{(2)}}(\lambda)$ is isotropic.

For the equality of norm distributions we will use a property of \textit{Cauchy-linear interpolation}: $g((1-\lambda)g^{-1}(\|\mathbf{Z}^{(1)}\|) + \lambda g^{-1}(\|\mathbf{Z}^{(2)}\|))$ is distributed identically to $\|\mathbf{Z}\|$. Thus norm of $f_{\mathbf{Z}^{(1)},\mathbf{Z}^{(2)}}(\lambda)$ is distributed equally to $\|\mathbf{Z}\|$.

\end{proof}

Figure~\ref{fig:2d_inters} shows comparison of Cauchy-linear and spherical Cauchy-linear interpolations on 2D plane for data points sampled from different distributions. Figure~\ref{fig:local_comp} shows the \textit{smoothness} of Cauchy-linear interpolation and a comparison between all the aforementioned interpolations. We also compare the data samples decoded from the interpolations by the DCGAN model trained on the CelebA dataset; results are shown on Figure~\ref{fig:inters}.

\begin{figure}[htb]
\begin{subfigure}{.25\textwidth}
  \centering
  \includegraphics[width=\linewidth]{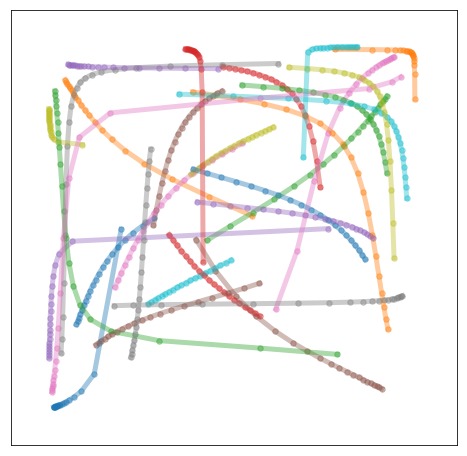}
  \caption{Uniform}
  \label{fig:uniform_to_cauchy}
\end{subfigure}%
\begin{subfigure}{.25\textwidth}
  \centering
  \includegraphics[width=\linewidth]{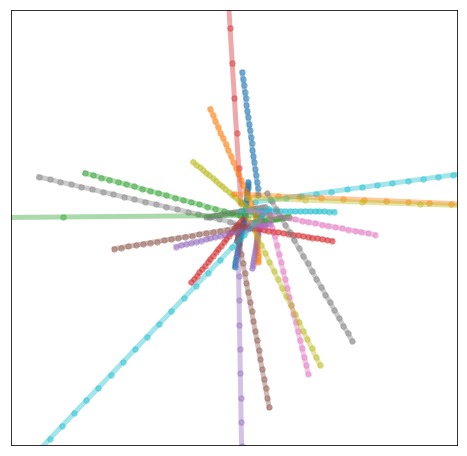}
  \caption{Cauchy}
  \label{fig:cauchy_to_cauchy}
\end{subfigure}%
\begin{subfigure}{.25\textwidth}
  \centering
  \includegraphics[width=\linewidth]{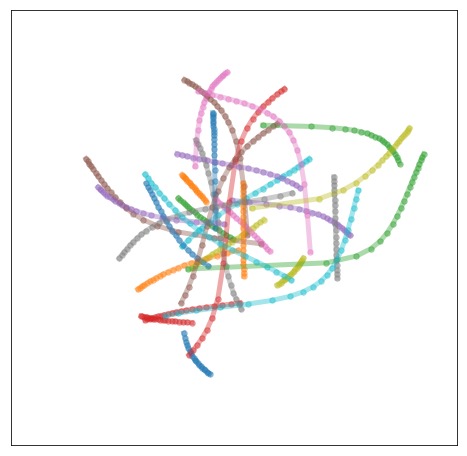}
  \caption{Normal}
  \label{fig:gauss_to_cauchy}
\end{subfigure}%
\begin{subfigure}{.25\textwidth}
  \centering
  \includegraphics[width=\linewidth]{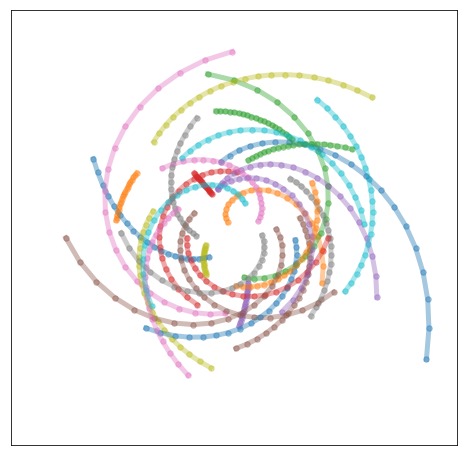}
  \caption{Normal - Spherical}
  \label{fig:normal_to_spherical_cauchy}
\end{subfigure}%
\caption{Visual comparison of Cauchy-linear interpolation (a, b, and c) for points sampled from different distributions and spherical Cauchy-linear (d). }
\label{fig:2d_inters}
\end{figure}

\begin{figure}[htb]
\centering
\begin{subfigure}{.4\textwidth}
  \centering
  \includegraphics[width=\linewidth]{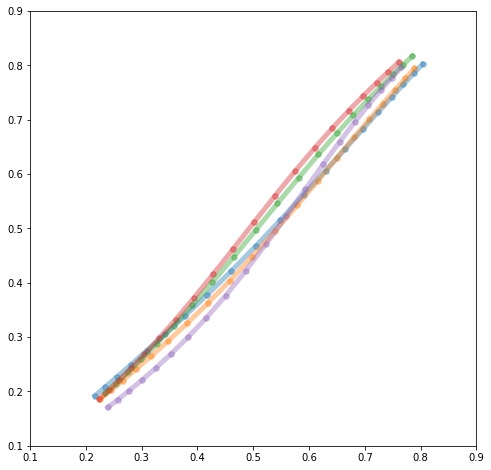}
  \caption{}
  \label{fig:local_cauchy}
\end{subfigure}%
\begin{subfigure}{.4\textwidth}
  \centering
  \includegraphics[width=\linewidth]{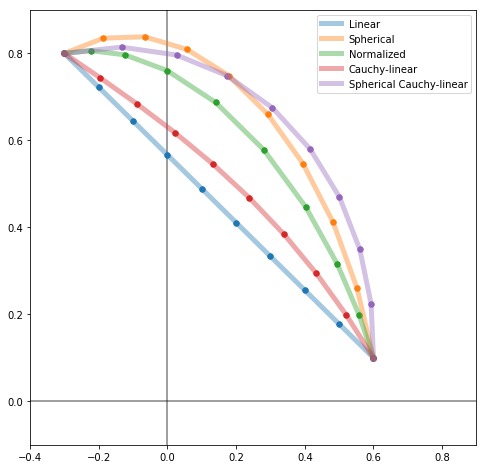}
  \caption{}
  \label{fig:compare_cauchy}
\end{subfigure}

\caption{Visual representation of interpolation property 1 shown for Cauchy-linear interpolation (a) and comparison between all considered interpolation methods for two points sampled from a normal distribution (b) on a 2D plane.}
\label{fig:local_comp}
\end{figure}

We will briefly list the conclusion of this chapter. Firstly, linear interpolation is the best choice if one does not care about the fourth property, i.e. interpolation samples being distributed identically to the end-points. Secondly, for every continuous distribution with additional assumptions we can define an interpolation that has a consistent distribution between endpoints and mid-samples, which will not satisfy property 2, i.e. will not be the shortest path in Euclidean space. Lastly, there exist distributions for which linear interpolation satisfies the fourth property, but those distributions cannot have finite mean.

To combine conclusions from the last two chapters: in our opinion there is a clear direction in which one would search for prior distributions for generative models. Namely, choosing those distribution for which the linear interpolation would satisfy all four properties listed above. On the other hand, in this chapter we have shown that if one would rather stick to the more popular prior distributions, it is fairly simple to define a nonlinear interpolation that would have consistent distributions between endpoints and midpoints.

\begin{figure}[htb]
  \centering
  \includegraphics[width=0.9\textwidth]{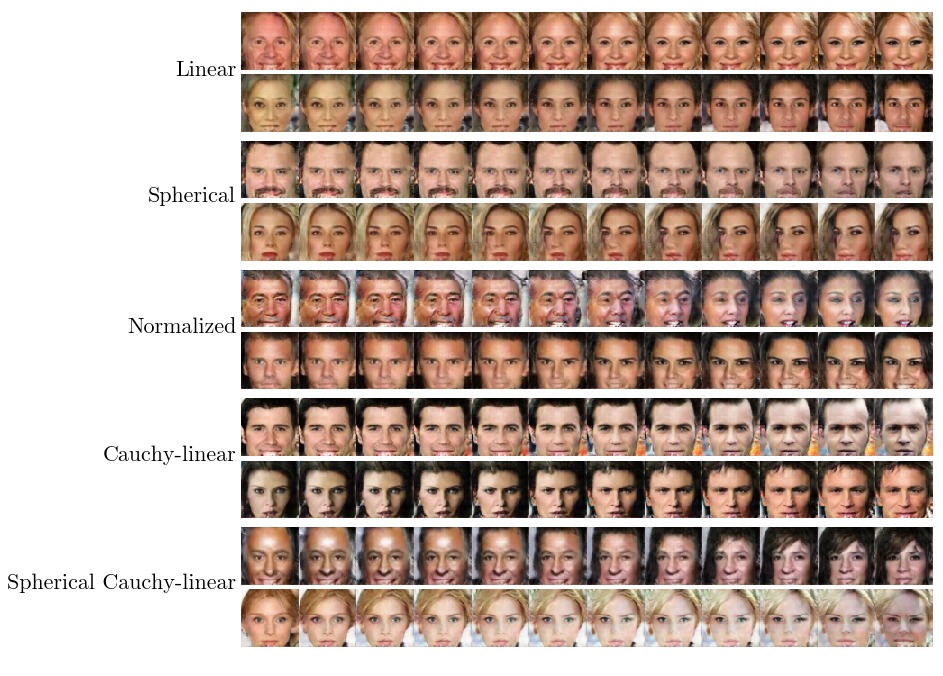}
  \caption{Different interpolations on latent space of a GAN model trained on standard Normal distribution.}
  \label{fig:inters}
\end{figure}

\section{Filling the Void}

In this section we investigate the claim that in close proximity to the origin of the latent space, generative models will generate unrealistic or faulty data~\cite{kilcher2017semantic}. We have tried different experimental settings and were somewhat unsuccessful in replicating this phenomena. Results of our experiments are summarized in Figure~\ref{fig:trials}. Even in higher latent space dimensionality $D=200$, the DCGAN model trained on the CelebA dataset was able to generate a face-like images, although with high amount of noise. We investigated this result furthermore and empirically concluded that the effect of \textit{filling} the origin of latent space emerges during late epochs of training. Figure~\ref{fig:progress} shows linear interpolations through origin of the latent space throughout the training process.   

\begin{figure}[htb]
  \centering
  \includegraphics[width=0.75\textwidth]{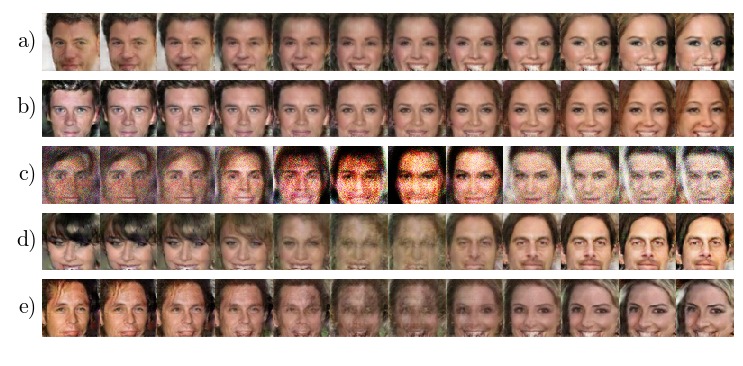}
  \caption{Linear interpolations through origin of latent space for different experimental setups: a) uniform noise distribution on $[-1, 1]^D$, b) uniform noise distribution on a sphere $\mathcal{S}^D$, c) fully connected layers, d) normal noise distribution with latent space dimensionality $D = 150$, e) Normal noise distribution with latent space dimensionality $D = 200$}
  \label{fig:trials}
\end{figure}

\begin{figure}[htb]
  \centering
  \includegraphics[width=0.75\textwidth]{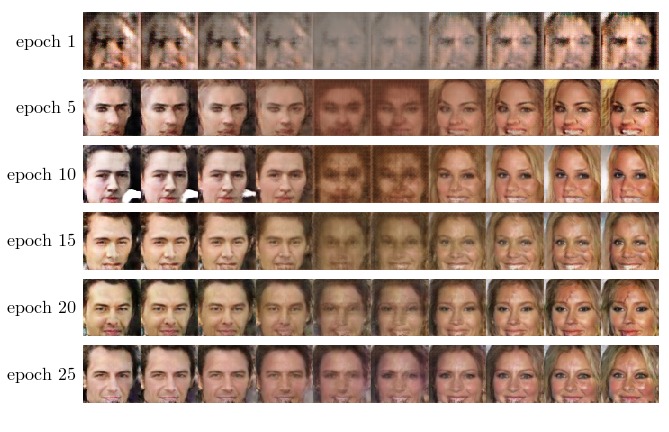}
  \caption{Emergence of \textit{sensible} samples decoded near the origin of latent space throughout the training process demonstrated in interpolations between opposite points from latent space.}
  \label{fig:progress}
\end{figure}

Data samples generated from samples located strictly \textit{inside} the high-probability-mass sphere may not be identically distributed as the samples used in training, but from the decoded result they seem to be on the manifold of data. On the other hand we observed that data generated using latent vectors with high norm, i.e. far \textit{outside} the sphere, are unrealistic. This might be due to the architecture, specifically the exclusive use of \textit{ReLU} activation function. Because of that, input vectors with large norms will result in abnormally high activation just before final saturating nonlinearity (usually \textit{tanh} or \textit{sigmoid} function), which in turn will make the decoded images highly color-saturated. It seems unsurprising that the exact value of the biggest \textit{sensible} norm of latent samples is related to norm of latent vectors seen during training.

The only possible explanation of the fact that model is able to generate sensible data from out-of-distribution samples is a very strong model prior, both architecture and training algorithm. We decided to empirically test strength of this prior in the experiment described below.

 \begin{figure}[htb]
\centering
\begin{subfigure}{0.75\textwidth}
  \includegraphics[width=\linewidth]{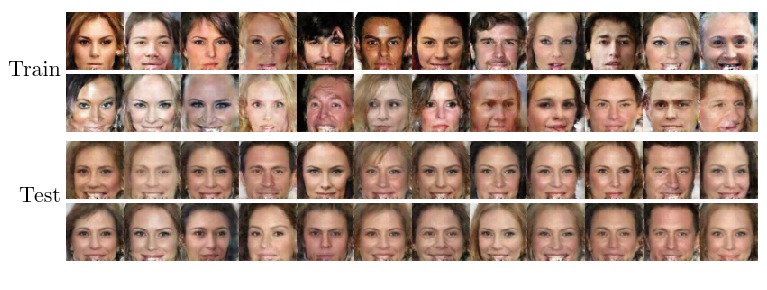}
  \vspace{-25pt}
  \caption{Samples from \textit{train} and \textit{test} distributions.}
  \label{fig:empty_test}
\end{subfigure}
\begin{subfigure}{.75\textwidth}
  \includegraphics[width=\linewidth]{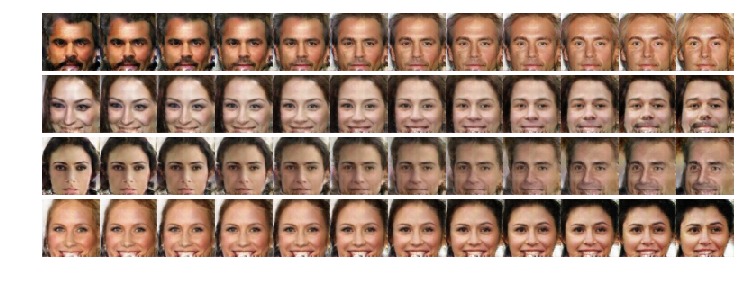}
  \vspace{-25pt}
  \caption{Linear interpolations between random points from \textit{train} distribution.}
  \label{fig:empty_inters}
\end{subfigure}
\caption{Images generated from a DCGAN model trained on 100-dimensional normal distribution. Train samples (from the distribution used in training) and test samples (same as train distribution but with lower variance) (a) and interpolations between random points from the training distribution~(b).}
\label{fig:train_test}
\end{figure}

We trained a DCGAN model on the CelebA dataset using a set of different noise distributions, all of which should suffer from the aforementioned \textit{empty} region in the origin of latent space. Afterwards, using those models, we generate data images decoded from out-of-distribution samples. We would not except to generate sensible images, as those latent samples should have never been seen during training. We visualize samples decoded from inside of the high-probability-mass sphere and linear interpolations traversing through it. We test a few different prior distributions: normal distribution $\mathcal{N}(\mathbf{0}, \mathbf{I})$, uniform distribution on hypercube $\mathcal{U}(-1,1)$, uniform distribution on sphere $\mathcal{S}(\mathbf{0},1)$, Discrete uniform distribution on set $\{-1, 1\}^D$. Experiment result are shown in Figure~\ref{fig:train_test}  with more in the appendix~\ref{cha:app_exp}.
 
\begin{figure}[htb]
\centering
\begin{subfigure}{0.75\textwidth}
  \includegraphics[width=\linewidth]{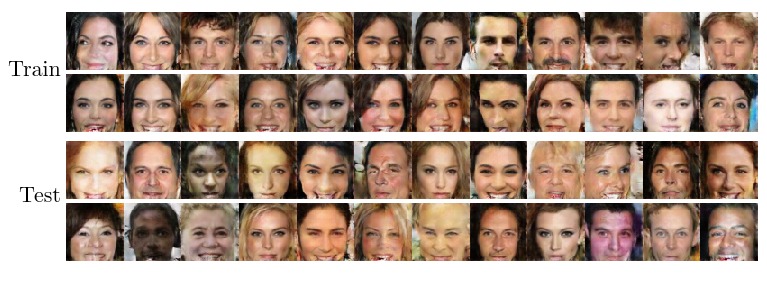}
  \vspace{-25pt}
  \caption{Samples from \textit{train} and \textit{test} distributions.}
  \label{fig:sparse_samples}
\end{subfigure}
\begin{subfigure}{.75\textwidth}
  \includegraphics[width=\linewidth]{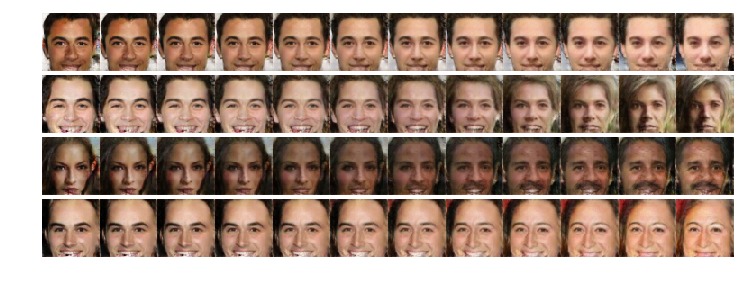}
  \vspace{-25pt}
  \caption{Linear interpolations between random points from \textit{train} distribution.}
  \label{fig:sparse_interns}
\end{subfigure}
\caption{Images generated from a DCGAN model trained on sparse 100-dimensional normal distribution with $K=50$. Train samples (from distribution used in training) and test samples (samples from a \textit{dense} normal distributions with adjust norm) (a) and interpolations between random points from the training distribution~(b).}
\label{fig:sparse_celeb}
\end{figure}

\begin{figure}[htb]
  \centering
  \includegraphics[width=0.6\textwidth]{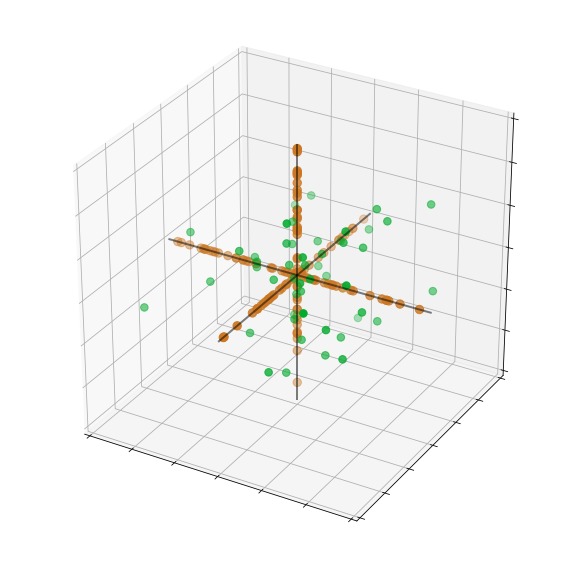}
  \caption{Visualization of the sparse training distribution for $D=3, K=1$, brown points being sampled from the train distribution, green from the test distribution.}
  \label{fig:sparse_2d}
\end{figure}

It might be that the origin of the latent space is \textit{surrounded} by the prior distribution, which may be the reason to generate sensible data from it. Thus we decided to test if the model still works if we train it on an explicitly \emph{sparse} distribution. We designed a pathological prior distribution in which after sampling from a given distribution, e.g. $\mathcal{N}(\mathbf{0}, \mathbf{I})$, we randomly draw $D-K$ coordinates and set them all to zero. Figure~\ref{fig:sparse_2d} shows samples from such a distributions in 3-dimensional case. Again, we trained the DCGAN model and generated images using latent samples from the \textit{dense} distribution, multiplying them beforehand by 
$\sqrt{K/D}$ to keep the norms consistent with those used in training. Results are shown in Figure~\ref{fig:sparse_celeb} with more in the appendix.

Lastly, we wanted to check if the model is capable of generating sensible data from regions of latent space that are completely orthogonal to those seen during training. We created another \textit{pathological} prior distribution, in which, after sampling from a given distribution (e.g. $\mathcal{N}(\mathbf{0}, \mathbf{I})$) we set the $D-K$ last coordinates of each point to zero. As before, we trained the DCGAN model and generated images using samples from original distribution but this time with first $K$ dimensions set to zero. Results are shown in Figure~\ref{fig:subspace_celeb} with more in the appendix.

\begin{figure}[htb]
  \centering
  \includegraphics[width=0.7\textwidth]{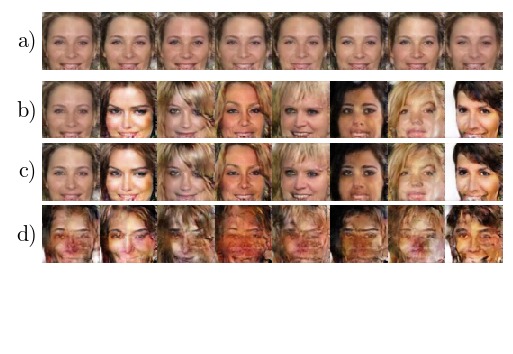}
  \vspace{-40pt}
  \caption{Images decoded from the DCGAN model trained on a \textit{pathological} 100-dimensional prior distribution with 50 last coordinates set to zero for different latent samples: a) orthogonal test distribution with 50 first coordinates set to 0, b) train distribution c) samples from train and test added element-wise  d) samples from test multiplied by the factor of 3 and train added element-wise.}
  \label{fig:subspace_celeb}
\end{figure}

To conclude our experiments we briefly remark on the results. For the first experiment, images generated from the test distribution are clearly meaningful despite being decoded from samples from low-probability-mass regions. One thing to note is the fact that they are clearly less diverse than those from training distribution.   

Decoded images from our second experiment are in our opinion similar enough between the train and test distribution to conclude that the DCGAN model does not suffer from training on a sparse latent distribution and is able to generalize without any additional mechanisms.

Out last experiment shows that while the DCGAN is still able to generate sensible images from regions orthogonal to the space seen during training, those regions still impact generative power and may lead to unrealistic data if they increase the activation in the network enough.

We observed that problems with \textit{hollow} latent space and linear interpolations might be caused by stopping the training early or using a weak model architecture. This leads us to conclusion that one needs to be very careful when comparing effectiveness of different latent probability distributions and interpolation methods.

\section{Summary}

We investigated the properties of multidimensional probability distributions in context of latent noise distribution of generative models. Especially we looked for pairs of distribution-interpolation, where the distributions of interpolation endpoints and midpoints are identical. 

We have shown that using $D$-dimensional Cauchy distribution as latent probability distribution makes linear interpolations between any number of latent points hold that consistency property. We have also shown that for popular priors with finite mean, it is impossible to have linear interpolations that will satisfy the above property. We argue that one can fairly easily find a non-linear interpolation that will satisfy this property, which makes a search for such interpolation less in\-te\-res\-ting. Those results are formal and should work for every generative model with fixed distribution on latent space. Although, as we have shown, Cauchy distribution comes with few useful theoretical properties, it is still perfectly fine to use normal or uniform distribution, as long as the model is powerful enough.

We also observed empirically that DCGANs, if trained long enough, are capable of generating sensible data from latent samples coming out-of-distribution. We have tested several pathological cases of latent priors to give a glimpse of what the model is capable of. At this moment we are unable to explain this phenomena and point to it as a very interesting future work direction.

\FloatBarrier

\bibliographystyle{plainnat}
\bibliography{bibliography}

\newpage
\begin{appendices}

\section{Experimental setup}

All experiments are run using DCGAN model, the generator network consists of a linear layer with 8192 neurons, follow by four convolution transposition layers, each using $5\times5$ filters and strides of 2 with number of filters in order of layers: 256, 128, 64, 3. Except the output layer where \textit{tanh} function activation is used, all previous layers use \textit{ReLU}. Discriminator's architecture mirrors the one from the generator with a single exception of using \textit{leaky ReLU} instead of vanilla \textit{ReLU} function for all except the last layer. No batch normalization is used in both networks. Adam optimizer with learning rate of $2e^{-4}$ and momentum set to $0.5$. Batch size 64 is used throughout all experiments. If not explicitly stated otherwise, latent space dimension is 100 and the noise is sampled from a multidimensional normal distribution $\mathcal{N}(\mathbf{0}, \mathbf{I})$. For the CelebA dataset we resize the input images to $64\times64$. The code to reproduce all our experiments is available at: coming soon!    

\section{Cauchy distribution - samples and interpolations}\label{cha:app_cauchy}

\begin{figure}[H]
  \centering
  \includegraphics[width=\textwidth]{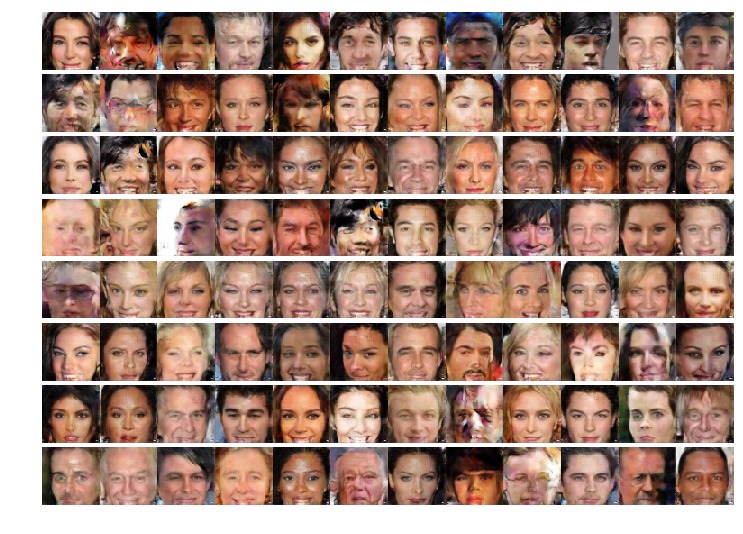}
  \caption{Generated images from samples from Cauchy distribution with occasional "failed" images from tails of the distribution.}
  \label{fig:app_1_samples}
\end{figure}

\begin{figure}[H]
  \centering
  \includegraphics[width=\textwidth]{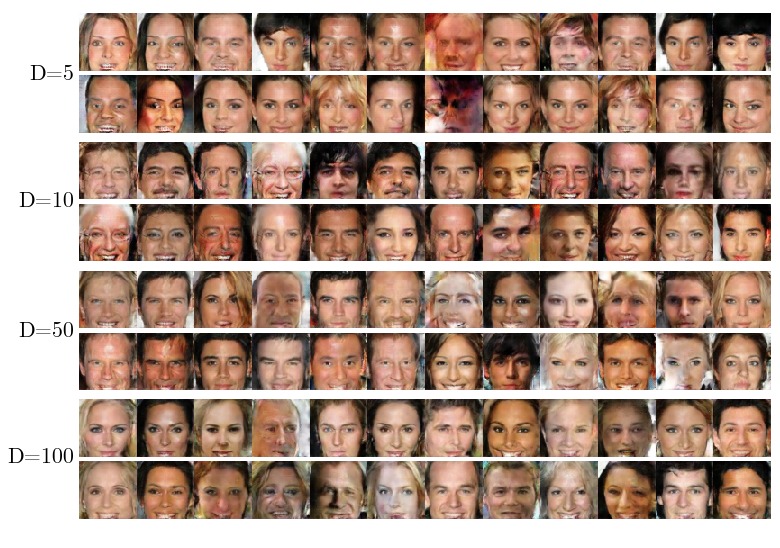}
  \caption{Generated images from samples from Cauchy distribution with different latent space dimensionality.}
  \label{fig:app_1_dim_samples}
\end{figure}

\begin{figure}[H]
  \centering
  \includegraphics[width=\textwidth]{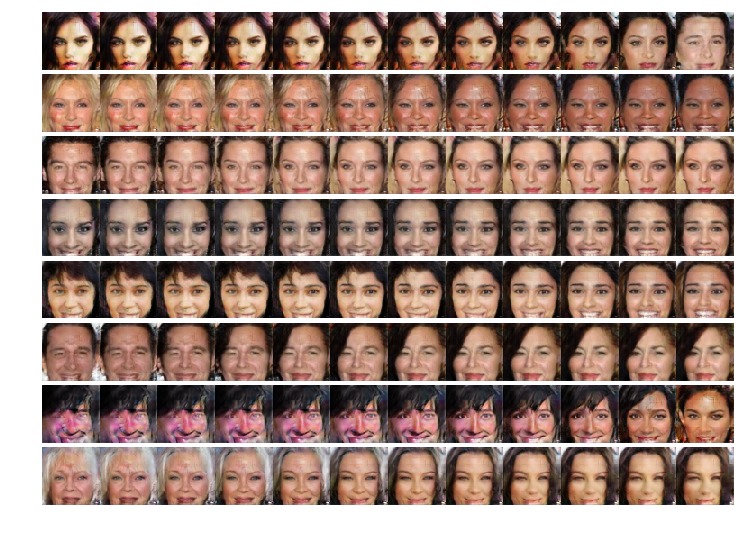}
  \caption{Linear interpolations between random points on a GAN trained on Cauchy distribution.}
  \label{fig:app_1_random_inters}
\end{figure}

\begin{figure}[H]
  \centering
  \includegraphics[width=\textwidth]{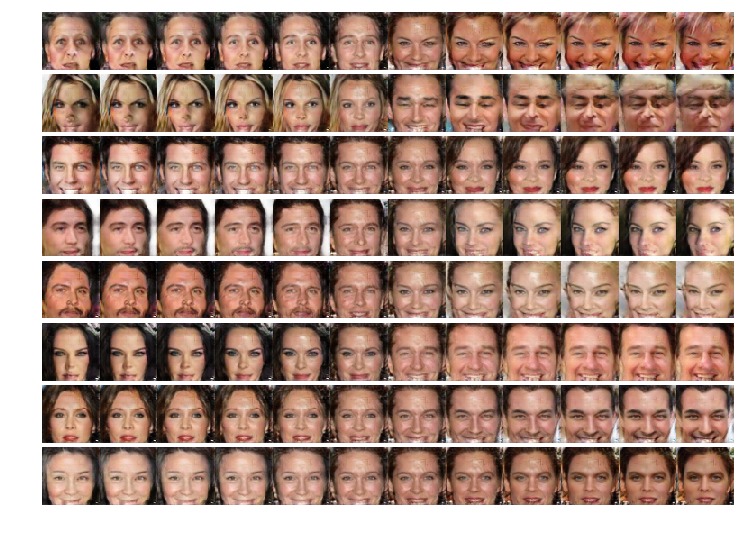}
  \caption{Linear interpolations between opposite points on a GAN trained on Cauchy distribution.}
  \label{fig:app_1_opp_inters}
\end{figure}

\begin{figure}[H]
  \centering
  \includegraphics[width=\textwidth]{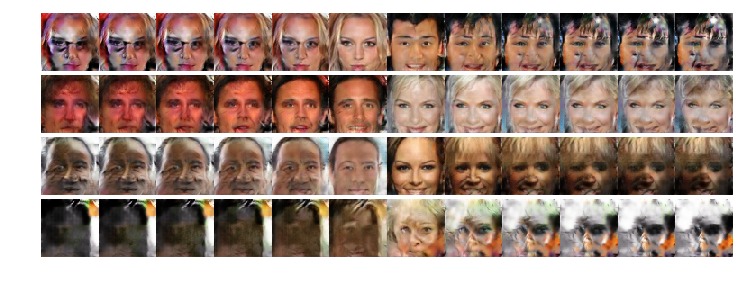}
  \caption{Linear interpolations between hand-picked points from tails of the Cauchy distribution.}
  \label{fig:app_1_failed}
\end{figure}

\section{More Cauchy-linear and spherical Cauchy-linear interpolations}\label{cha:app_inters}

\begin{figure}[H]
  \centering
  \includegraphics[width=0.9\textwidth]{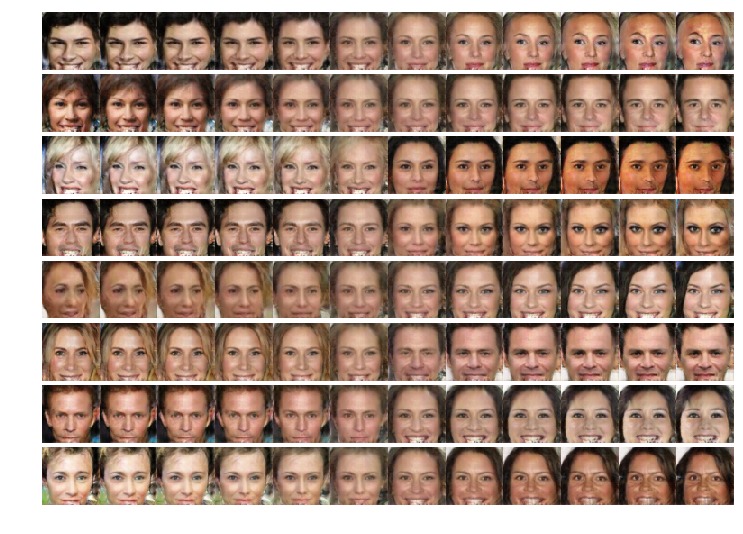}
  \caption{Cauchy-linear interpolations between opposite points on a GAN trained on Normal distribution.}
  \label{fig:app_2_cauchy_linear_opp}
\end{figure}

\begin{figure}[H]
  \centering
  \includegraphics[width=0.9\textwidth]{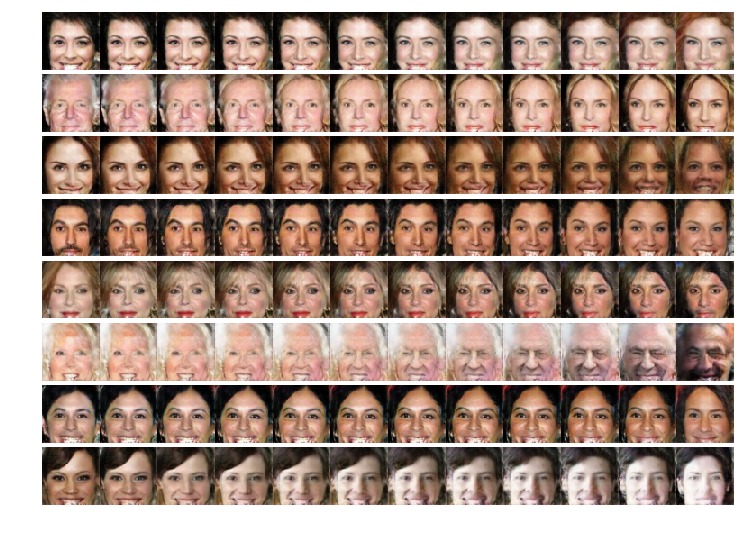}
  \caption{Cauchy-linear interpolations between random points on a GAN trained on Normal distribution.}
  \label{fig:app_2_cauchy_linear_rand}
\end{figure}

\begin{figure}[H]
  \centering
  \includegraphics[width=0.9\textwidth]{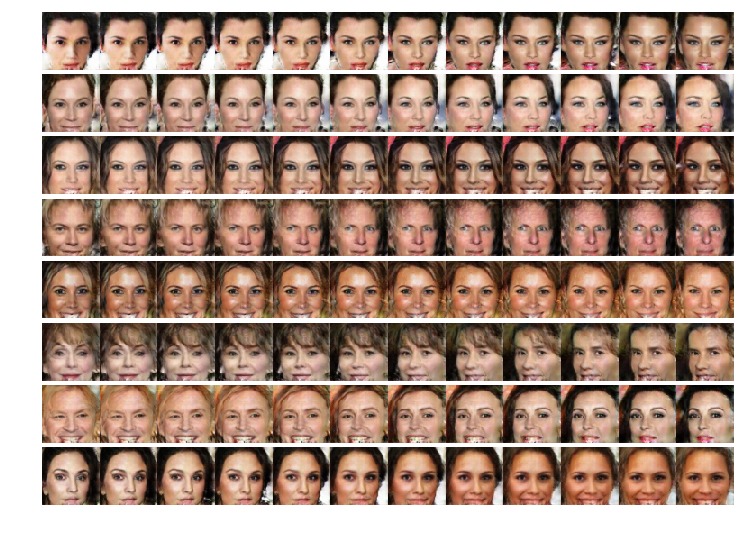}
  \caption{Spherical Cauchy-linear interpolations between random points on a GAN trained on Normal distribution.}
  \label{fig:app_2_sphcauchy_linear_rand}
\end{figure}

\section{More experiments with hollow latent space}\label{cha:app_exp}

 \begin{figure}[H]
\centering
\begin{subfigure}{0.75\textwidth}
  \includegraphics[width=\linewidth]{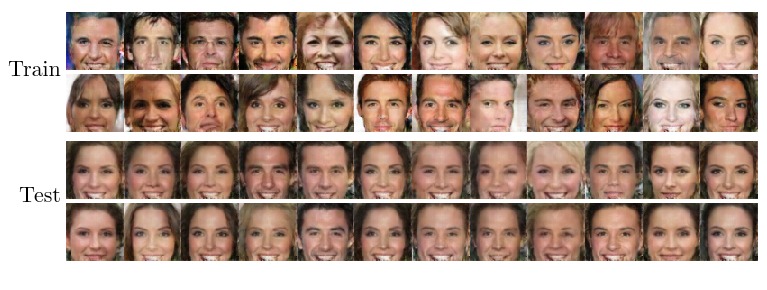}
  \vspace{-25pt}
  \caption{Samples from \textit{train} and \textit{test} distributions.}
  \label{fig:app_3_test_uni_samples}
\end{subfigure}
\begin{subfigure}{.75\textwidth}
  \includegraphics[width=\linewidth]{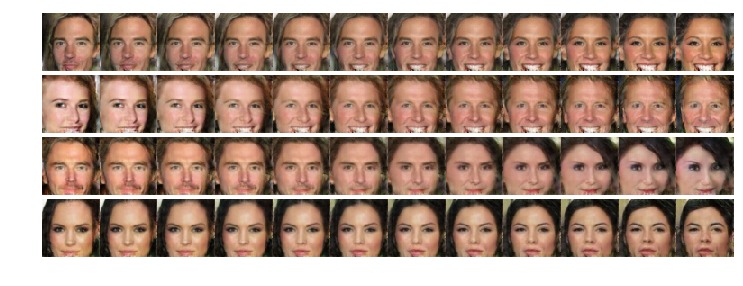}
  \vspace{-25pt}
  \caption{Linear interpolations between random points from \textit{train} distribution.}
  \label{fig:app_3_test_uni_interns}
\end{subfigure}
\caption{Images generated from a DCGAN model trained on 100-dimensional uniform distribution on hypercube $[-1,1]^{100}$. Train samples (from the distribution used in training) and test samples (same as train distribution but multiplied by $\alpha < 1$) (a) and interpolations between random points from the training distribution  (b).}
\label{fig:app_3_test_uni}
\end{figure}

 \begin{figure}[H]
\centering
\begin{subfigure}{0.75\textwidth}
  \includegraphics[width=\linewidth]{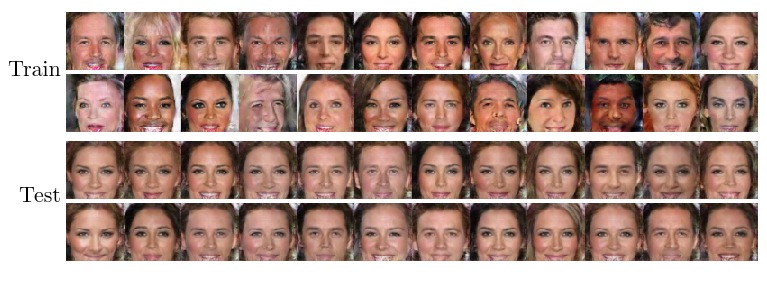}
  \vspace{-25pt}
  \caption{Samples from \textit{train} and \textit{test} distributions.}
  \label{fig:app_3_test_sphere_samples}
\end{subfigure}
\begin{subfigure}{.75\textwidth}
  \includegraphics[width=\linewidth]{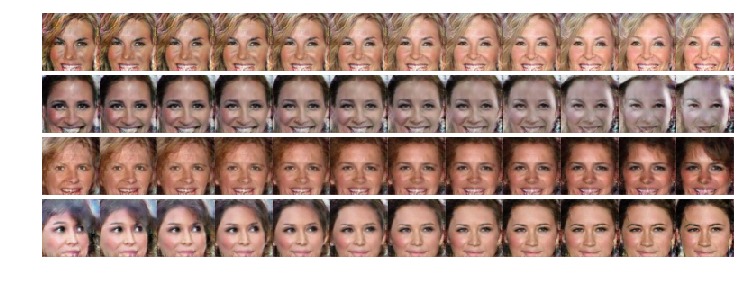}
  \vspace{-25pt}
  \caption{Linear interpolations between random points from \textit{train} distribution.}
  \label{figapp_3_test_sphere_interns}
\end{subfigure}
\caption{Images generated from a DCGAN model trained on 100-dimensional uniform distribution on sphere $\mathcal{S}^{100}$. Train samples (from the distribution used in training) and test samples (same as train distribution but multiplied by $\alpha < 1$) (a) and interpolations between random points from the training distribution  (b).}
\label{fig:app_3_test_sphere}
\end{figure}

 \begin{figure}[H]
\centering
\begin{subfigure}{0.75\textwidth}
  \includegraphics[width=\linewidth]{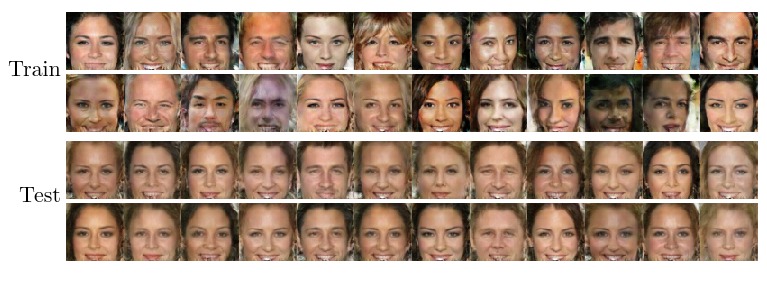}
  \vspace{-25pt}
  \caption{Samples from \textit{train} and \textit{test} distributions.}
  \label{fig:app_3_test_bi_samples}
\end{subfigure}
\begin{subfigure}{.75\textwidth}
  \includegraphics[width=\linewidth]{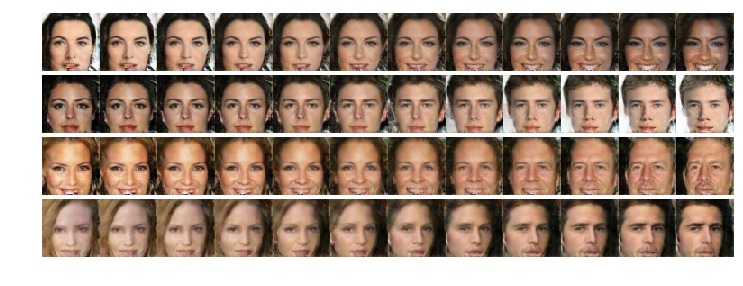}
  \vspace{-25pt}
  \caption{Linear interpolations between random points from \textit{train} distribution.}
  \label{fig:app_3_test_bi_inter}
\end{subfigure}
\caption{Images generated from a DCGAN model trained on 100-dimensional discrete distribution ${-1,1}^{100}$. Train samples (from the distribution used in training) and test samples (same as train distribution but multiplied by $\alpha < 1$) (a) and interpolations between random points from the training distribution  (b).}
\label{fig:app_3_test_bi}
\end{figure}

\end{appendices}

\end{document}